\documentclass[11pt]{amsart}
\def\isdraft{0}

\usepackage{txfonts}
\usepackage{graphicx}
\usepackage{amsaddr} 
\usepackage{mathtools}
\usepackage{amsmath,amssymb,amsfonts,dsfont}
\usepackage{prettyref} 
\usepackage[utf8]{inputenc}
\usepackage{enumitem}
\usepackage[hyphens]{url}
\usepackage{tikz-cd}
\usepackage{tikz}
\usetikzlibrary{positioning}
\usetikzlibrary{arrows}
\usepackage{bussproofs}
\usepackage[mathscr]{euscript}
\usepackage{hyperref} 
\usepackage{amsthm}
\usepackage{theapa}
\usepackage{a4wide}
\usepackage[color=black,textcolor=white\if\isdraft0,disable\fi]{todonotes}
\usepackage{stackengine}
\usepackage{stmaryrd}
\usepackage{wrapfig}
\usepackage{float}
\graphicspath{{fig/}}

\newtheorem{theorem}{Theorem}

\newtheorem{proposition}[theorem]{Proposition}

\theoremstyle{definition} 
\newtheorem{definition}[theorem]{Definition}
\newtheorem{notation}[theorem]{Notation}

\newtheorem{example}[theorem]{Example}

\newrefformat{cha}{Chapter \ref{#1}}
\newrefformat{sec}{Section \ref{#1}}
\newrefformat{tab}{Table \ref{#1}}
\newrefformat{fig}{Figure \ref{#1}}
\newrefformat{equ}{(\ref{#1})}
\newrefformat{app}{Appendix \ref{#1}}
\newrefformat{thm}{Theorem \ref{#1}}
\newrefformat{cor}{Corollary \ref{#1}}
\newrefformat{prop}{Proposition \ref{#1}}
\newrefformat{lem}{Lemma \ref{#1}}
\newrefformat{fact}{Fact \ref{#1}}
\newrefformat{obs}{Observation \ref{#1}}
\newrefformat{note}{Note \ref{#1}}
\newrefformat{idea}{Idea \ref{#1}}
\newrefformat{trivia}{Trivia \ref{#1}}
\newrefformat{def}{Definition \ref{#1}}
\newrefformat{not}{Notation \ref{#1}}
\newrefformat{con}{Convention \ref{#1}}
\newrefformat{rem}{Remark \ref{#1}}
\newrefformat{exa}{Example \ref{#1}}
\newrefformat{problem}{Problem \ref{#1}}
\newrefformat{claim}{Claim \ref{#1}}
\newrefformat{conjecture}{Conjecture \ref{#1}}
\newrefformat{exe}{Exercise \ref{#1}}
\newrefformat{alg}{Algorithm \ref{#1}}
\newrefformat{err}{Error \ref{#1}}
\newrefformat{que}{Question \ref{#1}}
\newrefformat{ite}{Item \ref{#1}}
\newrefformat{Q}{Q. \ref{#1}}
\newrefformat{warning}{Warning \ref{#1}}
\newrefformat{pseudocode}{Pseudocode \ref{#1}}

\title{
    On syntactically similar logic programs\\ and\\ sequential decompositions
}
\author{
    Christian Anti\'c
}
\address{
    christian.antic@icloud.com\\
    Vienna, Austria
}

\begin{document}
\begin{abstract} 
    Rule-based reasoning is an essential part of human intelligence prominently formalized in artificial intelligence research via logic programs. Describing complex objects as the composition of elementary ones is a common strategy in computer science and science in general. The author has recently introduced the sequential composition of logic programs in the context of logic-based analogical reasoning and learning in logic programming. Motivated by these applications, in this paper we construct a qualitative and algebraic notion of syntactic logic program similarity from sequential decompositions of programs. We then show how similarity can be used to answer queries across different domains via a one-step reduction. In a broader sense, this paper is a further step towards an algebraic theory of logic programming.
\end{abstract}
\maketitle

\section{Introduction}

Rule-based reasoning is an essential part of human intelligence prominently formalized in artificial intelligence research via logic programs with important applications to expert systems, database theory, and knowledge representation and reasoning  \cite<cf.>{Apt90,Lloyd87,Baral03}. Describing complex objects as the composition of elementary ones is a common strategy in computer science and science in general. \citeA{Antic23-23} has recently introduced the sequential composition of logic programs in the context of logic-based analogical reasoning and learning in logic programming. The sequential composition operation has been studied by \citeA{Antic21-1} in the propositional case and by \citeA{Antic21-2} in the non-monotonic case of answer set programming. \citeA{OKeefe85} is the first to study the composition of logic programs and in \citeA[§Related work]{Antic21-1} the author has compared O'Keefe's composition and the related composition of \citeA{Bugliesi94} and \citeA{Brogi92} to sequential composition. Other notable works dealing with composition are, for example, \citeA{Dong90,Dong95}, \citeA{Plambeck90a,Plambeck90a}, and \citeA{Ioannidis91}.

Motivated by the aforementioned applications to analogical reasoning, in this paper we study the sequential decomposition of logic programs and show that it gives rise to a qualitative and algebraic notion of syntactic logic program similarity. More precisely, we say that a program $P$ can be \textit{one-step reduced} to a program $R$---in symbols, $P\lesssim R$---iff $P=(Q\circ R)\circ S$ for some programs $Q$ (prefix) and $S$ (suffix). Now given a query $q$ to $P$, we can answer $q$ by translating every SLD-resolution step of $P$ via the prefix $Q$ into an SLD-resolution step of $R$ and translating the result back again with the suffix $S$ (cf. \prettyref{exa:Plus}). We then say that $P$ and $R$ are \textit{syntactically similar} iff $P$ can be one-step reduced to $R$ and vice versa, that is, iff $P\lesssim R$ and $R\lesssim P$. This definition has appealing mathematical properties. For example, we can show that the programs $Plus$ for the addition of numerals and the program $Append$ for list concatenation are syntactically similar according to our definition (cf. \prettyref{exa:Plus}). Interestingly, when translating queries to $Append$ into queries of the seemingly simpler program $Plus$, we obtain `entangled' terms of the form $s([\;\,])$ and $s([b,c])$ which are neither numerals nor lists, and we believe that `entangled' syntactic objects like these are essential for cross-domain analogical reasoning (see the discussion in \prettyref{sec:Conclusion}). Syntactic similarity is related to the number of bound variables per rule of a program and we deduce from this fact that, for example, the program $Member$ for checking list membership, containing maximally two bound variables per rule, can be one-step reduced to the program $Append$ containing maximally three bound variables per rule, but not vice versa: $Member<Append$ (cf. Examples \ref{exa:Member} and \ref{exa:Member2}). This is interesting as the prefix $Q$ witnessing $Member=Q\circ Append$ resembles the mapping between $Member$ and $Append$ of \citeA{Tausend91} introduced for analogical reasoning.

In a broader sense, this paper is a further step towards an algebraic theory of logic programming.

\section{Logic Programs}

We recall the syntax and semantics of logic programs by mainly following the lines of \citeA{Apt90}.

\subsection{Syntax}\label{sec:Syntax}

An (\textit{unranked first-order}) \textit{language} $L$ is defined as usual from predicate, function, and constant symbols and variables. Terms and atoms over $L$ are defined in the usual way. We denote the set of all ground $L$-atoms by $HB_L$ or simply by $HB$ called the \textit{Herbrand base} over $L$, and we denote the set of all $L$-atoms not containing function or constant symbols by $XHB_L$. Substitutions and (most general) unifiers of terms and (sets of) atoms are defined as usual. We call any bijective substitution mapping variables to variables a \textit{renaming}.

Let $L$ be a language. A (\textit{Horn logic}) \textit{program} over $L$ is a set of \textit{rules} of the form
\begin{align}\label{equ:r} 
    A_0\leftarrow A_1,\ldots,A_k,\quad k\geq 0,
\end{align} where $A_0,\ldots,A_k$ are $L$-atoms. It will be convenient to define, for a rule $r$ of the form \prettyref{equ:r}, $h(r):=\{A_0\}$ and $b(r):=\{A_1,\ldots,A_k\}$, extended to programs by $h(P):=\bigcup_{r\in P}h(r)$ and $b(P):=\bigcup_{r\in P}b(r)$. In this case, the \textit{size} of $r$ is $k$ denoted by $sz(r)$. A \textit{fact} is a rule with empty body and a \textit{proper rule} is a rule which is not a fact. We denote the facts and proper rules in $P$ by $facts(P)$ and $proper(P)$, respectively. A program $P$ is \textit{ground} if it contains no variables and we denote the grounding of $P$ which contains all ground instances of the rules in $P$ by $gnd(P)$. We call a ground program \textit{propositional} if it contains only propositional atoms with no arguments. The set of all \textit{variants} of $P$ is defined by $variants(P):=\bigcup_{\theta\text{ renaming}}P[\theta]$. The variants of a program will be needed to avoid clashes of variables in the definition of composition. A variable is \textit{bound} in a rule if it appears in the head and body. Define the \textit{dual} of $P$ by
\begin{align*} 
    P^d:=facts(P)\cup\{A\leftarrow h(r)\mid r\in proper(P): A\in b(r)\}.
\end{align*} Roughly, we obtain the dual of a theory by reversing all the arrows of its proper rules.


\subsection{Semantics}

An \textit{interpretation} is any set of ground atoms. We define the \textit{entailment relation}, for every interpretation $I$, inductively as follows: (i) for an atom $a$, $I\models a$ if $a\in I$; (ii) for a set of atoms $B$, $I\models B$ if $B\subseteq I$; (iii) for a rule $r$ of the form \prettyref{equ:r}, $I\models r$ if $I\models b(r)$ implies $I\models h(r)$; and, finally, (iv) for a program $P$, $I\models P$ if $I\models r$ holds for each rule $r\in P$. In case $I\models P$, we call $I$ a \textit{model} of $P$. The set of all models of $P$ has a least element with respect to set inclusion called the \textit{least model} of $P$. We say that $P$ and $R$ are \textit{logically equivalent} if their least models coincide. Define the \textit{van Emden-Kowalski operator} of $P$, for every interpretation $I$, by
\begin{align*} 
    T_P(I):=\{h(r)\mid r\in gnd(P):I\models b(r)\}.
\end{align*} It is well-known that an interpretation $I$ is a model of $P$ iff $I$ is a prefixed point of $T_P$ and the least model of $P$ coincides with the least fixed point of $T_P$.

We define the \textit{left} and \textit{right reduct} of $P$, with respect to some interpretation $I$, respectively by
\begin{align*} 
    ^IP:=\{r\in P\mid I\models h(r)\} \quad\text{and}\quad P^I:=\{r\in P\mid I\models b(r)\}.
\end{align*}





\subsection{SLD-Resolution}

Logic programs compute via a restricted form of resolution, called \textit{SLD-resolution}, as follows. For simplicity, we consider here only the ground case. Let $q$ be a ground query $\leftarrow A_1,\ldots,A_k$, $k\geq 1$, and suppose that for some $i$, $1\leq i\leq k$ and $m\geq 0$, $r=A_i\leftarrow A'_1,\ldots,A'_m$ is a rule from $gnd(P)$. Then $q'$ given by
\begin{align*} \leftarrow A_1,\ldots,A_{i-1},A'_1,\ldots,A'_m,A_{i+1},\ldots,A_k
\end{align*} is called a \textit{resolvent} of $q$ and $r$, and $A_i$ is called the \textit{selected atom} of $q$. By iterating this process we obtain a sequence of resolvents which is called an \textit{SLD-derivation}. A derivation can be finite or infinite. If its last query is empty then we speak of an \textit{SLD-refutation} of the original query $q$. In this case we have derived an \textit{SLD-proof} of $A_1,\ldots,A_k$. A \textit{failed} SLD-derivation is a finite derivation which is not a refutation. In case $A$ is a ground atom with an SLD-proof from $P$, we say that $A$ is an \textit{SLD-consequence} of $P$ and write $P\vdash A$. 
For a rule $r$ of the form \prettyref{equ:r}, we write $P\vdash r$ in case $P\vdash A_0$ whenever $P\vdash A_i$ holds for every $1\leq i\leq k$. We denote the \textit{empty query} by $\square$.

\section{Composition}

In this section, we recall the sequential composition of logic programs as defined by \citeA{Antic23-23} and studied in the propositional case by \citeA{Antic21-1}.

\begin{notation} In the rest of the paper, $P$ and $R$ denote logic programs over some joint language $L$.
\end{notation}

The rule-like structure of logic programs induces naturally a compositional structure as follows \cite[Definition 4]{Antic23-23}.

We define the (\textit{sequential}) \textit{composition} of $P$ and $R$ by\footnote{We write $X\subseteq_k Y$ in case $X$ is a subset of $Y$ consisting of $k$ elements.}
\begin{align*} 
    P\circ R:=\left\{h(r\vartheta)\leftarrow b(S\vartheta) \;\middle|\; 
    \begin{array}{l}
        r\in P\\
        S\subseteq_{sz(r)}variants(R)\\
        h(S\vartheta)=b(r\vartheta)\\
        \vartheta=mgu(b(r),h(S))
    \end{array}
    \right\}.
\end{align*}
\todo[inline]{}

Roughly, we obtain the composition of $P$ and $R$ by resolving all body atoms in $P$ with the `matching' rule heads of $R$. This is illustrated in the next example, where we construct the even from the natural numbers via composition.

\begin{example}\label{exa:Even} Consider the program
\begin{align*} Nat:= \left\{
\begin{array}{l}
    nat(0)\\
    nat(s(x))\leftarrow nat(x)
\end{array}
\right\}
\end{align*} generating the natural numbers. By composing the only proper rule in $Nat$ with itself, we obtain
\begin{align*} \{nat(&s(x))\leftarrow nat(x)\}\circ\{nat(s(x))\leftarrow nat(x)\}=\{nat(s(s(x)))\leftarrow nat(x)\}.
\end{align*} Notice that this program, together with the single fact in $Nat$, generates the \textit{even} numbers. 
\end{example}

Notice that we can reformulate sequential composition as
\begin{align}\label{equ:bigcup} P\circ R=\bigcup_{r\in P}(\{r\}\circ R),
\end{align} which directly implies right-distributivity of composition, that is,
\begin{align}\label{equ:(P_cup_Q)_circ_R} (P\cup Q)\circ R=(P\circ R)\cup (Q\circ R)\quad\text{holds for all propositional Horn theories }P,Q,R.
\end{align} However, the following counter-example shows that left-distributivity fails in general:  
\begin{align*} \{a\leftarrow b,c\}\circ(\{b\}\cup\{c\})=\{a\} \quad\text{and}\quad (\{a\leftarrow b,c\}\circ\{b\})\cup(\{a\leftarrow b,c\}\circ\{c\})=\emptyset.
\end{align*}

Define the \textit{unit program} by the Krom program\footnote{Recall from \prettyref{sec:Syntax} that $XHB_L$ consists of all $L$-atoms not containing function or constant symbols.}
\begin{align*} 
    1_L:=\{A\leftarrow A\mid A\in XHB_L\}.
\end{align*} In the sequel, we will omit the reference to $L$.

The space of all logic programs over some fixed language is closed under sequential composition with the neutral element given by the unit program and the empty program serves as a left zero, that is, we have
\begin{align*} 
    P\circ 1=1\circ P=1 \quad\text{and}\quad \emptyset\circ P=\emptyset.
\end{align*}

We can simulate the van Emden-Kowalski operator on a syntactic level without any explicit reference to operators via sequential composition, that is, for any interpretation $I$, we have
\begin{align}\label{equ:T_P} T_P(I)=gnd(P)\circ I.
\end{align}

As facts are preserved by composition and since we cannot add body atoms to facts via composition on the right, we have
\begin{align}\label{equ:IP=I} I\circ P=I.
\end{align}

\subsection{Ground Programs}

Ground programs are possibly infinite propositional programs where ground atoms can have a fixed inner structure, which means that we can import here the results of \citeA{Antic21-1} which do not depend on finiteness.

\begin{notation} In the rest of this subsection, $P$ and $R$ denote ground programs.
\end{notation}

For ground programs, sequential composition simplifies to
\begin{align*} 
    P\circ R=\{h(r)\leftarrow b(S)\mid r\in P,S\subseteq_{sz(r)}R:h(S)=b(r)\}.
\end{align*}

Our first observation is that we can compute the heads and bodies of a ground program $P$ via
\begin{align}\label{equ:h(P)} h(P)=P\circ HB \quad\text{and}\quad b(P)=proper(P)^d\circ HB.
\end{align} Moreover, we have
\begin{align}\label{equ:h(PR)} h(PR)\subseteq h(P) \quad\text{and}\quad b(PR)\subseteq b(R).
\end{align}

Given an interpretation $I$, we define
\begin{align*} I^\ominus:=1^{HB-I}\cup I \quad\text{and}\quad I^\oplus:=\{A\leftarrow(\{A\}\cup I)\mid A\in HB\}.
\end{align*} It is not difficult to show that $PI^\ominus$ is the program $P$ where all occurrences of the ground atoms in $I$ are removed from the rule bodies in $P$, that is, we have
\begin{align*} PI^\ominus&=\{h(r)\leftarrow (b(r)-I)\mid r\in P\}.
\end{align*} Similarly, $PI^\oplus$ is the program $P$ with the ground atoms in $I$ added to the rule bodies of all \textit{proper} rules in $P$, that is, we have
\begin{align*} PI^\oplus&=facts(P)\cup\{h(r)\leftarrow (b(r)\cup I)\mid r\in proper(P)\}.
\end{align*}

The left and right reducts can be represented via composition and the unit program by
\begin{align}\label{equ:^IP} ^IP=1^I\circ P \quad\text{and}\quad P^I=P\circ 1^I.
\end{align}

Interestingly enough, we can represent the grounding of a (non-ground) program $P$ via composition with the grounding of the unit program by
\begin{align}\label{equ:gnd(P)} gnd(P)=(gnd(1)\circ P)\circ gnd(1).
\end{align}

\section{Decomposition and Similarity}

Let us start this section with a definition of a qualitative and algebraic notion of syntactic logic program similarity in terms of sequential decomposition.

Given two logic programs $P$ and $R$, we define
\begin{align*} 
    P\lesssim R\quad:\Leftrightarrow\quad\text{there exist programs $Q$ and $S$ such that $P=(QR)S$.}
\end{align*} In that case, we say that $P$ can be \textit{one-step reduced} to $R$, and we call $Q$ a \textit{prefix} and $S$ a \textit{suffix} of $P$. In case $P\lesssim R$ and $R\lesssim P$, we say that $P$ and $R$ are (\textit{syntactically}) \textit{similar} denoted by $P\approx R$.

\todo[inline]{prove that $\lesssim$ is transitive}

The computational interpretation of $P=(QR)S$ is that we can answer a query to $P$ by translating it via $Q$ into a query of $R$, computing one SLD-resolution step with $R$, and translating it back to $P$ via $S$ as is demonstrated in the next example.


\begin{example}\label{exa:Plus} Consider the programs
\begin{align*} Plus:= \left\{
\begin{array}{l}
    plus(0,y,y)\\
    plus(s(x),y,s(z))\leftarrow\\
    \qquad plus(x,y,z)
\end{array}
\right\} \quad\text{and}\quad 
Append := \left\{
\begin{array}{l}
    append([\;\,],y,y)\\
    append([u\mid x],y,[u\mid z])\leftarrow\\
    \qquad plus(x,y,z)
\end{array}
\right\}
\end{align*} implementing the addition of positive numbers (represented as numerals) and the concatenation of lists, respectively. To show that $Plus$ and $Append$ are syntactically similar, we define the programs
\begin{align*} Q:= \left\{
\begin{array}{l}
  append([\;\,],y,y)\leftarrow plus(0,y,y)\\
  append([u\mid x],y,[v\mid z])\leftarrow plus(s(x),y,s(z))
\end{array}
\right\}
\end{align*} and
\begin{align*} S:=\{plus(x,y,z)\leftarrow append(x,y,z)\},
\end{align*} and compute
\begin{align*} Append=(Q\circ Plus)\circ S \quad\text{and}\quad Plus=(Q^d\circ Append)\circ S^d.
\end{align*} This shows:
\begin{align*} Plus\approx Append.
\end{align*}

The following SLD-derivation demonstrates how we can use the prefix $Q$ and suffix $S$ to append two lists via the seemingly simpler program for the addition of numerals:
\begin{align*}
&\xleftarrow{\;\;\;?\;\;\;} append([a],[b,c],[a,b,c])\\
&\xleftarrow{\;\;Q\;\;\;} plus(s([\;\,]),[b,c],s([b,c]))\\
&\xleftarrow{Plus} plus([\;\,],[b,c],[b,c])\\
&\xleftarrow{\;\;S\;\;\;} append([\;\,],[b,c],[b,c])\\
&\xleftarrow{\;\;Q\;\;\;} plus(0,[b,c],[b,c])\\
&\xleftarrow{Plus} \square.
\end{align*} This shows, via a translation to numerals,
\begin{align*} Append\vdash append([a],[b,c],[a,b,c]).
\end{align*} Interestingly, the SLD-derivation above contains the `entangled' terms $s([\;\,])$ and $s([a,b])$ which are neither numerals nor lists, and we believe that such `entangled' syntactic objects---which under the conventional doctrine of programming are not well-typed---are characteristic for reasoning across different domains and deserve special attention (this is discussed in \prettyref{sec:Conclusion}).
\end{example}

\begin{example}\label{exa:Member} \citeA{Tausend91} derive a mapping between the program $Append$ from \prettyref{exa:Plus} above and the program
\begin{align*} Member:= \left\{
\begin{array}{l}
    member(u,[u\mid x]),\\
    member(u,[v\mid x])\leftarrow\\
    \qquad member(u,x)
\end{array}
\right\},
\end{align*} which computes list membership. We can now ask---by analogy to \citeA{Tausend91}---whether $Member$ can be one-step reduced to $Append$ (and vice versa) according to our definition. Define the programs
\begin{align*} Q:= \left\{
\begin{array}{l}
    member(u,[u\mid x]),\\
    member(u,[v\mid x])\leftarrow append([v\mid x],u,[v\mid x])
\end{array}
\right\}
\end{align*} and
\begin{align*} S:=\{append(x,y,z)\leftarrow member(y,x)\}.
\end{align*} It is not hard to compute
\begin{align*} Member=(Q\circ Append)\circ S,
\end{align*} which shows that membership can indeed be one-step reduced to the program for appending lists:
\begin{align*} Member\lesssim Append.
\end{align*} Interestingly enough, the programs $Q$ and $S$, which in combination permute the first two arguments of $Append$ and `forget' about the last one, resemble the mapping computed in \citeA{Tausend91} by other means. The converse fails, roughly, since $Member$ does not contain enough `syntactic structure' to represent $Append$. The intuitive reason is that $Member$ has only two bound variables, whereas $Append$ has three (see \prettyref{exa:Member2}).
\end{example}

\section{Properties of Similarity}

Given ground atoms $A_0,\ldots,A_k$, the identity
\begin{align*} \{A_0\}=\{A_0\leftarrow A_1\}\{A_1\}
\end{align*} shows
\begin{align*} \{A_0\}\lesssim \{A_0\leftarrow A_1\}.
\end{align*} Since we cannot add body atoms to facts via composition on the right, we have
\begin{align*} \{A_0\leftarrow A_1\}\not\lesssim\{A_0\}.
\end{align*} Hence,
\begin{align*} \{A_0\}<\{A_0\leftarrow A_1\}.
\end{align*} On the other hand, we have
\begin{align*} 
    \{A_0\leftarrow A_1\}&=\{A_0\leftarrow A_1,\ldots,A_k\}\{A_2,\ldots,A_k\}^\ominus,\\
    \{A_0\leftarrow A_1,\ldots,A_k\}&=\{A_0\leftarrow A_1\}\{A_2,\ldots,A_k\}^\oplus,
\end{align*} which shows
\begin{align*} \{A_0\leftarrow A_1\}\approx\{A_0\leftarrow A_1,\ldots,A_k\}.
\end{align*}

In the non-ground case, we have, for example:
\begin{align*} \{p\leftarrow p\}<\{p(x_1)\leftarrow p(x_1)\}<\{p(x_1,x_2)\leftarrow p(x_1,x_2)\}<\ldots.
\end{align*} This motivates the following definition.

\begin{definition} The \textit{width} of a rule is given by the number of its bound variables, extended to programs via $width(P):=\max_{r\in P}width(r)$.
\end{definition}

The number of bound variables cannot increase via composition---for example:
\begin{align*} (\{p(x,y)\leftarrow p(x)\}\circ\{p(x)\leftarrow p(x)\})\circ\{p(x)\leftarrow p(x,y)\}=\{p(x,y)\leftarrow p(x,z)\}.
\end{align*} This entails
\begin{align*} width(PR)\leq width(P) \quad\text{and}\quad width(PR)\leq width(R).
\end{align*} Hence, we have
\begin{align}\label{equ:width} P\lesssim R \quad\Rightarrow\quad width(P)\leq width(R).
\end{align}

\begin{example}\label{exa:Member2} Reconsider the programs $Append$ and $Member$ of \prettyref{exa:Member}. We have
\begin{align*} width(Member)=2 \quad\text{whereas}\quad width(Append)=3.
\end{align*} By \prettyref{equ:width} we thus have $Append\not\lesssim Member$.
\end{example}

The following propositions summarize some facts about syntactic similarity.

\begin{proposition} For any program $P$ and interpretations $I$ and $J$, we have
\begin{align} 
    gnd(P)&\lesssim P \quad\text{and}\quad T_P(I)\lesssim gnd(P),\\
    I&\lesssim P \quad\text{and}\quad P\cup I\lesssim P,\\
    \label{equ:I_approx_J} I&\approx J.
\end{align} Moreover, we have
\begin{align*} 
    P\approx I \quad&\Leftrightarrow\quad \text{$P$ is an interpretation}.
\end{align*}
\end{proposition}
\begin{proof} The relations in the first line are immediate consequences of \prettyref{equ:gnd(P)} and \prettyref{equ:T_P}, respectively. The relation $I\lesssim P$ follows from \prettyref{equ:IP=I}. The computation
\begin{align*} P\cup I\stackrel{\prettyref{equ:IP=I}}=P\cup IP\stackrel{\prettyref{equ:(P_cup_Q)_circ_R}}=(1\cup I)P
\end{align*} shows $P\cup I\lesssim P$. The similarity $I\approx J$ follows from $I\lesssim P$. The last similarity follows from \prettyref{equ:I_approx_J} together with
\begin{align*} P\lesssim I \quad\Leftrightarrow\quad P=(QI)S\stackrel{\prettyref{equ:IP=I}}=QI\stackrel{\prettyref{equ:T_P}}=T_Q(I),\text{ for some $Q,S$} \quad\Rightarrow\quad \text{$P$ is an interpretation}.
\end{align*}
\end{proof}

\begin{proposition} For any \textit{ground} program $P$ and interpretation $I$, we have
\begin{align}
    h(P)&\lesssim P \quad\text{and}\quad b(P)\lesssim proper(P)^d,\\
    facts(P)&\lesssim P \quad\text{and}\quad P\lesssim 1\cup facts(P),\\
    \label{equ:PI^oplus} PI^\oplus&\approx P\quad\text{and}\quad PI^\ominus\lesssim P,\\
    \label{equ:^IP_lesssim_}^IP&\lesssim P \quad\text{and}\quad P^I\lesssim P.
\end{align} Moreover, we have
\begin{align} 
    \label{equ:ominus} P\approx PI^\ominus \quad&\Leftrightarrow\quad facts(P)=facts(PI^\ominus),\\
    \label{equ:facts} P\approx facts(P) \quad&\Leftrightarrow\quad P=facts(P)
\end{align}
\end{proposition}
\begin{proof} The relations in the first line are immediate consequences of \prettyref{equ:h(P)}. The relation $facts(P)\lesssim P$ follows from $facts(P)=P\circ\emptyset$, and the relation $P\lesssim 1\cup facts(P)$ follows from
\begin{align*} 
    (1\cup facts(P))proper(P)&\stackrel{\prettyref{equ:(P_cup_Q)_circ_R}}=proper(P)\cup facts(P)proper(P)\\
    &\stackrel{\prettyref{equ:IP=I}}=proper(P)\cup facts(P)\\
    &=P.
\end{align*} The relation $PI^\oplus\lesssim P$ holds trivially; $P\lesssim PI^\oplus$ follows from $P=(PI^\oplus)I^\ominus$. Similarly, the relation $PI^\ominus\lesssim P$ holds trivially; for $P\lesssim PI^\oplus$ see \prettyref{equ:ominus}. The relations in \prettyref{equ:^IP_lesssim_} are immediate consequences of \prettyref{equ:^IP}.

Next, we prove \prettyref{equ:ominus}: $facts(P)=facts(PI^\ominus)$ means that by removing the ground atoms in $I$ from all rule bodies in $P$ we do not obtain novel facts---hence, we can add the ground atoms from $I$ back to the rule bodies of $P$ via $P=(PI^\ominus)I^\oplus$ which shows $P\lesssim PI^\ominus$ and see \prettyref{equ:PI^oplus} (recall that we cannot add body atoms to facts via composition); the other direction is analogous. 

Finally, the equivalence in \prettyref{equ:facts} is shown as follows: we have $P\lesssim facts(P)$ iff 
\begin{align}\label{equ:Q_circ_facts(P)} P=(Q\circ facts(P))\circ S=Q\circ facts(P),\quad\text{for some $Q$ and $S$,}
\end{align} since
\begin{align*} Q\circ facts(P)\stackrel{\prettyref{equ:T_P}}=T_Q(facts(P))
\end{align*} is an interpretation and
\begin{align*} T_Q(facts(P))\circ S\stackrel{\prettyref{equ:IP=I}}=T_Q(facts(P)).
\end{align*} The identity \prettyref{equ:Q_circ_facts(P)} holds iff $P=T_Q(facts(P))$ which is equivalent to $P=facts(P)$ since $T_Q(facts(P))$ yields an interpretation.
\end{proof}

The following characterization of syntactic similarity follows immediately from the definitions.

\begin{proposition}\label{prop:lesssim} For any programs $P$ and $R$, we have $P\lesssim R$ iff for each rule $r\in P$ there is a rule $s_r$, a subset $R_r$ of $R$ with $width(r)\leq width(R_r)$, and a program $S_r$ such that
\begin{align*} \{r\}=(\{s_r\}R_r)S_r \quad\text{and}\quad (\{s_r\}R)S\subseteq P,\quad\text{where $S:=\bigcup_{r\in P}S_r$.}
\end{align*} In this case, we have $P=(QR)S$ with $Q:=\bigcup_{r\in P}\{s_r\}$.
\end{proposition}

\begin{example} Consider the propositional programs
\begin{align*} P=\left\{
\begin{array}{l}
    c\\
    a\leftarrow b,c\\
    b\leftarrow a,c
\end{array}
\right\} \quad\text{and}\quad \pi_{(a\,b)}=\left\{
\begin{array}{l}
    a\leftarrow b\\
    b\leftarrow a
\end{array}
\right\}.
\end{align*} We construct the programs $Q$ and $S$ such that $P=(Q\pi_{(a\,b)})S$ according to \prettyref{prop:lesssim}. Define
\begin{align*} 
    &r_1:=c \quad\Rightarrow\quad s_{r_1}:=c \quad\text{and}\quad S_{r_1}:=\emptyset,\\
    &r_2:=a\leftarrow b,c \quad\Rightarrow\quad s_{r_2}:=a\leftarrow a \quad\text{and}\quad S_{r_2}:=\{b\leftarrow b,c\},\\
    &r_3:=b\leftarrow a,c \quad\Rightarrow\quad s_{r_3}:=b\leftarrow b \quad\text{and}\quad S_{r_3}:=\{a\leftarrow a,c\},
\end{align*} and
\begin{align*} 
    Q:=\{s_{r_1},s_{r_2},s_{r_3}\}=1^{\{a,b\}}\cup\{c\} \quad\text{and}\quad S:=S_{r_1}\cup S_{r_2}\cup S_{r_3}=\{c\}^\oplus-1^{\{c\}}.
\end{align*} This yields
\begin{align*} 
    P=((1^{\{a,b\}}\cup\{c\})\pi_{(a\,b)})(\{c\}^\oplus-1^{\{c\}}).
\end{align*} Similar computations yield
\begin{align*} 
    \pi_{(a\,b)}=(1^{\{a,b\}}P)\{c\}^\ominus.
\end{align*} This shows
\begin{align*} 
    P\approx\pi_{(a\,b)}.
\end{align*}
\end{example}

\begin{example} Consider the propositional programs
\begin{align*} 
    \pi_{(a\,b)}:= \left\{
    \begin{array}{l}
        a\leftarrow b\\
        b\leftarrow a
    \end{array}
    \right\} \quad\text{and}\quad R:= \left\{
    \begin{array}{l}
        a\leftarrow b\\
        b\leftarrow b
    \end{array}
    \right\}.
\end{align*} We have
\begin{align*} 
    R=\pi_{(a\,b)}R \quad\Rightarrow\quad R\lesssim \pi_{(a\,b)}.
\end{align*} On the other hand, there can be no programs $Q$ and $S$ such that $\pi_{(a\,b)}=(QR)S$ since we cannot rewrite the rule body $b$ of $R$ into $a$ and $b$ simultaneously via composition on the right. This shows
\begin{align*} 
    R<\pi_{(a\,b)}.
\end{align*}
\end{example}

The following simple example shows that syntactic similarity and logical equivalence are `orthogonal' concepts.

\begin{example} The empty program is logically equivalent with respect to the least model semantics to the propositional program $a\leftarrow a$ consisting of a single rule. Since we cannot obtain the rule $a\leftarrow a$ from the empty program via composition, logical equivalence does not imply syntactic similarity. For the other direction, the computations
\begin{align*} \left\{
\begin{array}{l}
    a\\
    b\leftarrow a
\end{array}
\right\}= \left\{
\begin{array}{l}
    a\\
    b\leftarrow a,b
\end{array}
\right\}\{b\}^\ominus \quad\text{and}\quad \left\{
\begin{array}{l}
    a\\
    b\leftarrow a,b
\end{array}
\right\}=\left\{
\begin{array}{l}
    a\\
    b\leftarrow a
\end{array}
\right\}\{b\}^\oplus
\end{align*} show that the programs $P=\{a,\;b\leftarrow a\}$ and $R=\{a,\;b\leftarrow a,b\}$ are syntactically similar; however, $P$ and $R$ are not logically equivalent.
\end{example}

\todo[inline]{bsp das zeigt, das programme mit ganz unterschiedlicher semantik syntaktisch aehnlich sein koennen}

\section{Conclusion}\label{sec:Conclusion}

This paper showed how a qualitative and algebraic notion of syntactic logic program similarity can be constructed from the sequential decomposition of programs. We derived some elementary properties of syntactic similarity and, more importantly, demonstrated how it can be used to answer queries across different domains by giving some illustrative examples. Interestingly, in \prettyref{exa:Plus}, in the process of answering queries in the list domain by translating it to the seemingly simpler domain of numerals, we obtained `entangled' terms of the form $s([\;\,])$ and $s([b,c])$ which are neither numerals nor lists. We believe that `entangled' syntactic objects of this form---which under the conventional doctrine of programming are not well-typed (see, for example, the remarks in \cite[§5.8.2]{Apt97})---are characteristic for `creative' analogies across different domains and deserve special attention. More precisely, if we think of (logic) programs as tools for solving problems, then using those tools `creatively' often requires, as in the example mentioned above, a tight coupling between seemingly incompatible objects. However, if those objects are statically typed by the programmer, then this might prevent the creation and exploitation of useful analogical `bridges' for knowledge transfer.

\bibliographystyle{theapa}
\bibliography{/Users/christianantic/Bibdesk/Bibliography,/Users/christianantic/Bibdesk/Preprints,/Users/christianantic/Bibdesk/Publications}
\end{document}